\def\year{2019}\relax

\documentclass[letterpaper]{article} %DO NOT CHANGE THIS

\usepackage{aaai19}  %Required
\usepackage{times}  %Required
\usepackage{helvet}  %Required
\usepackage{courier}  %Required
\usepackage{url}  %Required
\usepackage{graphicx}  %Required

\usepackage[mathscr]{euscript}
\usepackage{amsmath, amsthm, amssymb, amsfonts, mathtools, bm}
\usepackage{multicol}

\newtheorem{axiom}{Axiom}
\newtheorem{lemma}{Lemma}
\newtheorem*{assumption}{Assumption}
\newtheorem{theorem}{Theorem}
\newtheorem*{corollary}{Corollary}

\newcommand{\inlinecite}[1]{\citeauthor{#1} \shortcite{#1}}
\newcommand{\Epsilon}{\scalebox{1.6}{$\mathbf{\epsilon}$}}

\title{Rethinking the Discount Factor in Reinforcement Learning:\\A Decision Theoretic Approach}
\author{Silviu Pitis\\University of Toronto, Vector Institute\\Toronto, ON\\spitis@cs.toronto.edu}

\begin{document}
\maketitle

\begin{abstract}
Reinforcement learning (RL) agents have traditionally been tasked with maximizing the value function of a Markov decision process (MDP), either in continuous settings, with fixed discount factor $\gamma < 1$, or in episodic settings, with $\gamma = 1$. While this has proven effective for specific tasks with well-defined objectives (e.g., games), it has never been established that fixed discounting is suitable for general purpose use (e.g., as a model of human preferences). This paper characterizes rationality in sequential decision making using a set of seven axioms and arrives at a form of discounting that generalizes traditional fixed discounting. In particular, our framework admits a state-action dependent ``discount'' factor that is not constrained to be less than 1, so long as there is eventual long run discounting. Although this broadens the range of possible preference structures in continuous settings, we show that there exists a unique ``optimizing MDP'' with fixed $\gamma < 1$ whose optimal value function matches the true utility of the optimal policy, and we quantify the difference between value and utility for suboptimal policies. Our work can be seen as providing a normative justification for (a slight generalization of) Martha White's RL task formalism (2017) and other recent departures from the traditional RL, and is relevant to task specification in RL, inverse RL and preference-based RL.
\end{abstract}

\section{Introduction} \label{section_introduction}

The field of reinforcement learning (RL) studies agents that learn to interact ``optimally'' with a partially known environment \cite{sutton1998reinforcement}. The common environmental model, underlying almost all analysis in RL, is the Markov decision process (MDP), with optimal behavior defined in terms of the expected sum of future rewards, discounted at a constant rate per time step. The motivation for this approach has historically been a practical one; e.g., to enable us to make ``precise theoretical statements'' \cite{sutton1998reinforcement} or because rewards provide ``the most succinct, robust, and transferable definition of a task'' \cite{abbeel2004apprenticeship}. While this has proven effective for tasks with well-defined objectives (e.g., games), modern RL aims to tackle increasingly general tasks, and it has never been established that the MDP is suitable for general purpose use. Recently, researchers have adopted more flexible value functions with state, state-action, and transition dependent discount factors \cite{sutton2011horde,silver2017ThePE,white2017UnifyingTS}. Practical considerations aside, which approach, if any, is most sensible from a \textit{normative} perspective?

We address this question by presenting an axiomatic framework to characterize rationality in sequential decision making. After a motivating example in Section \ref{section_motivation} and a discussion of related work in Section \ref{section_related}, including an overview of the relevant work in normative decision theory, we proceed in Section \ref{section_theory} to develop our framework. From a set of seven axioms we derive a form of utility function that admits a state-action dependent discount factor that can be greater than 1, so long as there is long run discounting. Although this broadens the range of possible preference structures (vis-\`a-vis the fixed discount setting), we show that there exists a unique ``optimizing MDP'' with fixed $\gamma < 1$ whose optimal value function matches the utility of the optimal policy, and we quantify the difference between value and utility for suboptimal policies. Section \ref{section_GRL} discusses the implications of our work for task specification in RL, inverse RL and preference-based RL, and highlights an interesting connection to Dayan's successor representation \shortcite{dayan1993improving}.

\section{Motivation} \label{section_motivation}

\begin{figure*}[ht]
\includegraphics[width=\textwidth]{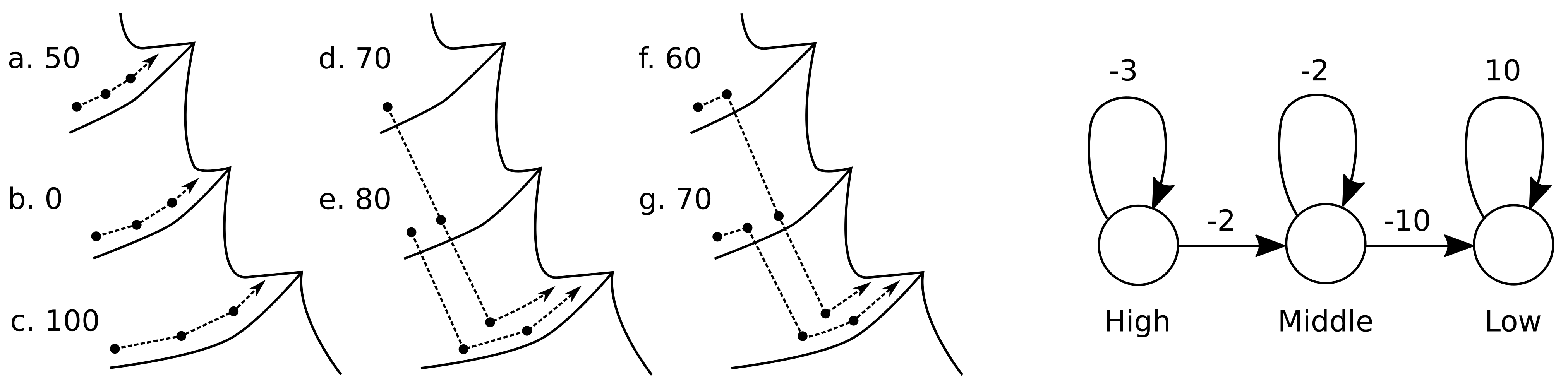}
\caption{The utilities in the Cliff Example (left) are not compatible with any 3-state MDP; e.g. the only MDP that satisfies the Bellman equation with respect to the optimal policy when $\gamma = 0.9$ (right).}
\label{fig_cliffMDP}
\end{figure*}

To motivate our work, we show how the traditional fixed $\gamma$ value function fails to model a set of preferences. Consider an agent that is to walk in a single direction on the side of a cliff forever (the ``Cliff Example''). Along the cliff are three parallel paths---low, middle and high---to which we assign utilities of 100, 0 and 50, respectively (Figure \ref{fig_cliffMDP} (left) a-c). Let us also suppose the agent has the option to jump down one level at a time from a higher path, but is unable to climb back up. Thus the agent has many options. Four of them are shown in Figure \ref{fig_cliffMDP} (left) d-g with their assigned utilities. 

At a glance, there does not appear to be anything irrational about the utility assignments. But it is impossible to represent this utility structure using a 3-state fixed $\gamma$ MDP with infinite time horizon. To see this, consider that for the Bellman equation to be satisfied with respect to the optimal policy, paths c-g in Figure \ref{fig_cliffMDP} (left) imply the reward values shown in Figure \ref{fig_cliffMDP} (right) when one assumes $\gamma = 0.9$. This implies that the utilities of paths a and b are -30 and -20, respectively. Not only is this incorrect, but the order is reversed! This holds for all $\gamma \in [0, 1)$ (see Supplement for a short proof). It follows that either the utility assignments are irrational, or the MDP structure used is inadequate. This inconsistency motivates a closer look: what characteristics, or axioms, should we demand of ``rational'' utility assignments? Are they satisfied here, and if so, what does this mean for our default modeling approach (the fixed $\gamma$ MDP)?

\section{Related work} \label{section_related}

\inlinecite{koopmans1960stationary} provided the first axiomatic development of discounted additive utility over time. This and several follow-up works are summarized and expanded upon by \inlinecite{koopmans1972representations} and \inlinecite{meyer1976preferences}. The applicability of these and other existing discounted additive utility frameworks to general purpose RL is limited in several respects. For instance, as remarked by \inlinecite{sobel2013discounting}, most axiomatic justifications for discounting have assumed deterministic outcomes, and only a handful of analyses address stochasticity  \cite{meyer1976preferences,epstein1983stationary,sobel2013discounting}. Naively packaging deterministic outcome streams into arbitrary lotteries (as suggested by \inlinecite{meyer1976preferences}, \S9.3) is difficult to interpret in case of control (see our commentary in Subsection \ref{subsection_prefs_over_prospects} on the resolution of intra-trajectory uncertainty) and entirely hypothetical since the agent never makes such choices---compare our (slightly) less hypothetical ``original position'' approach in Subsection \ref{subsection_cardinal}.

Existing frameworks have also typically been formulated with a focus on future streams of \textit{consumption} or \textit{income}, which has led to assumptions that do not necessarily apply to sequential decision making. When each unit in a stream is assumed to be scalar, this already rules out ``certain types of intertemporal complementarity'' \cite{diamond1965evaluation}. This complementarity is also ruled out by mutual independence assumptions \cite{koopmans1972representations} in frameworks admitting vector-valued units (see commentary of \inlinecite{meyer1976preferences}, \S 9.4, and \inlinecite{frederick2002time}, \S 3.3). Even seemingly innocuous assumptions like bounded utility over arbitrary outcome streams have important consequences. For instance, \inlinecite{meyer1976preferences} and \inlinecite{epstein1983stationary} derive similar utility representations as our Theorem \ref{theorem_functional_form}, but use the aforementioned assumption to conclude that the discount factor is always less than or equal to 1 (their argument is discussed in Subsection \ref{subsection_gamma}). %see the discussions following equation 9.22 in the former and equation 17 in the latter). 
By contrast, our framework admits a state-action dependent ``discount'' factor that may be greater than 1, \textit{so long as there is eventual long run discounting}---thus, specific, measure zero trajectories may have unbounded reward but the utility of a stochastic process that may produce such trajectories will still exist. This is illustrated in the Supplement. 

\inlinecite{frederick2002time} provide a comprehensive empirical review of the discounted additive utility model as it pertains to human behavior and conclude that it ``has little empirical support.'' While this is consistent with our normative position, it does not invalidate discounting, as humans are known to exhibit regular violations of rationality \cite{tversky1986rational}. Such violations are not surprising, but rather a necessary result of bounded rationality \cite{simon1972theories}. Our work is in a similar vein to \inlinecite{russell2014update} in that it argues that this boundedness necessitates new research directions.

Several papers in economics---including \inlinecite{kreps1977decision} (and sequels), \inlinecite{jaquette1976utility} (and prequels), \inlinecite{porteus1975optimality} (and sequel)---examine sequential decision processes that do not use an additive utility model. Of particular relevance to our work are \inlinecite{von1953theory}, \inlinecite{kreps1978temporal} and \inlinecite{sobel1975ordinal}, on which our axiomatic framework is based. To the authors' knowledge, no study, including this one, has ever provided a direct axiomatic justification of the MDP as a model for general rationality. This is not so surprising given that the MDP has historically been applied as a task-specific model. For example, the MDP in Bellman's classic study \shortcite{bellman1957markovian} arose ``in connection with an equipment replacement problem''.

The MDP is the predominant model used in RL research. It is commonly assumed that complex preference structures, including arbitrary human preferences, can be well represented by traditional MDP value functions \cite{abbeel2004apprenticeship,christiano2017deep}. To the authors' knowledge, the best (indeed, \textit{only}) theoretical justification for this assumption is due to \inlinecite{ng2000algorithms}, discussed in Subsection \ref{subsection_optimizing}. Some RL researchers have proposed generalizations of the traditional value functions that include a transition dependent discount factor \cite{white2017UnifyingTS,silver2017ThePE}. Our work can be understood as a normative justification for (a slight generalization of) these approaches. This, and further connections to the RL literature are discussed in Section \ref{section_GRL}. 

\section{Theory}\label{section_theory}

\subsection{Preliminaries}

Our basic environment is a sequential decision process (SDP) with infinite time horizon, formally defined as the tuple $(S, A, T, T_0)$ where $S$ is the state space, $A$ is the action space, $T: S \times A \to \mathscr{L}(S)$ is the transition function mapping state-action pairs to \textit{lotteries} (i.e., probability distributions with finite support) over next states, and $T_0 \in \mathscr{L}(S)$ is the distribution from which initial state $s_0$ is chosen. To accommodate episodic (terminating) processes, one may simply introduce a terminal state that repeats ad infinitum. 

A \textit{trajectory} is an infinite sequence of states and actions, $(s_t, a_t, s_{t+1}, a_{t+1}, \dots)$. For each non-negative integer $t$, $y_t \in Y_t$ denotes the \textit{history} from time $t = 0$ through the state at time $t$; e.g., $(s_0, a_0, s_1, a_1, s_2) \in Y_2$. $y_{t[i]}$ indexes the $i$-th state in $y_t$; e.g., $(s_0, a_0, s_1, a_1, s_2)_{[2]} = s_2$.

A (stochastic) \textit{stationary policy} $\pi: S \to \mathscr{L}(A)$ (or $\omega$ when a second generic is needed) maps states to lotteries over actions. A \textit{non-stationary policy} from time $t$, $\Pi_t = (\pi_t, \pi_{t+1}\ |\ y_{t+1}, \pi_{t+2}\ |\ y_{t+2}, \dots)$ (or just $\Pi$, or $\Omega$ when a second generic is needed) is a conditional sequence of stationary policies where the choice of $\pi_t$ may depend on $y_t$. $\Pi_{[i]}$ indexes $\Pi$'s $i$-th element (e.g., $\Pi_{t[1]} = \pi_{t+1}\ |\ y_{t+1}$), $\Pi_{[i:]}$ denotes $(\Pi_{[i]}$, $\Pi_{[i+1]}, \dots)$. $\Pi(s)$ is shorthand for $\Pi_{[0]}(s)$. Note that stationary policy $\pi$ may be viewed as the non-stationary policy $(\pi, \pi, \dots)$. The space of all non-stationary policies is denoted $\mathbf{\Pi}$. 

A Markov decision process (MDP) is an SDP together with a tuple $(R, \gamma)$, where $R: S \times A \to \mathbb{R}$ returns a bounded scalar reward for each transition and $\gamma \in [0, 1)$ is a discount factor. For a given MDP, we define the value function for a policy $\Pi$, $V^\Pi: S \to \mathbb{R}$, as $V^\Pi(s_t) = \mathbb{E}[\sum_{t=0}^\infty \gamma^t R(s_t, a_t)]$. Further, we define the Q-function for $\Pi$, $Q^\Pi: S \times A \to \mathbb{R}$ as $Q(s,a) = V^{a\Pi}(s)$, where $a\Pi = (\pi, \Pi_{[0]}, \Pi_{[1]}, \dots$) is the non-stationary policy that uses generic policy $\pi$ with $\pi(s) = a$ in the first step and follows policy $\Pi$ thereafter.

\subsection{Preferences over prospects}\label{subsection_prefs_over_prospects}

As a starting point in our characterization of rationality, we would like to apply the machinery of expected utility theory \cite{von1953theory} to preferences over the set of possible futures, or ``prospects'', $\mathcal{P}$, and more generally, lotteries on prospects, $\mathscr{L}(\mathcal{P})$. Some care is required in defining a prospect so as to satisfy the necessary axioms. In particular, we would like (strict) preference to be \textit{asymmetric}, meaning that between any two prospects $p, q \in \mathcal{P}$, at most one of $p \succ q$ or $q \succ p$ holds, where $\succ$ denotes strict preference (for convenience, we also define weak preference $p \succeq q$ as \smash{\underline{not}} $q \succ p$, and indifference $p \sim q$ as $p \succeq q$ and $q \succeq p$). Without additional assumptions, preferences over trajectories or policies fail to satisfy asymmetry. 

Suppose that preferences were defined over bare trajectories, as in ``preference-based RL'' \cite{wirth2017survey}. One problem with this is that intra-trajectory uncertainty has already been resolved; e.g., to evaluate a trajectory that risked, but did not realize, a visit to a difficult region of the state space, one must make an assumption about how the agent would have behaved in that region. More generally, it is unclear whether trajectories should be compared with respect to action quality (what could have happened) or with respect to outcomes (what did happen). 

Suppose instead that preferences were defined over bare policies. This is still problematic because preference would depend on the current state distribution (not just $T_0$). To see this take an SDP with disconnected state trees $S_1$ and $S_2$, where $\pi_1 \succ \pi_2$ in $S_1$ but $\pi_2 \succ \pi_1$  in $S_2$.

To avoid such difficulties, we define a prospect as a pair $(s, \Pi)$, where $s \in S$ and $\Pi$ is an arbitrary non-stationary policy, which represents the stochastic process that results when the agent starts in state $s$ and behaves according to $\Pi$ thereafter. For this definition to work we need the following assumption, in absence of which different histories leading up to the initial state $s$ could result in preference reversal:

\begin{assumption}[Markov preference, MP]
Preferences over prospects are independent of time $t$ and history $y_t$.
\end{assumption}

One might justify MP in several ways. First, \textit{trivially}, one may restrict the scope of inquiry to time $t = 0$. Second, \textit{theoretically}, it is a consequence of certain preference structures (e.g., the structure associated with the standard optimality criteria for MDPs). Third, \textit{practically}, one can view MP as a constraint on agent design. Typical RL agents, like the DQN agent of \inlinecite{mnih-atari-2013} are restricted in this way. Finally, \textit{constructively}, MP may be achieved by using an ``internal'' SDP that is derived from the environment SDP, as shown in Subsection \ref{subsection_constructiveMP}. The DRQN agent of \inlinecite{hausknecht2015drqn}, which augments the DQN agent with a recurrent connection, implements such a construction. 

As compared to trajectories, all uncertainty in prospects is left unresolved, which makes them comparable to the ``temporal lotteries'' of \inlinecite{kreps1978temporal}. As a result, it is admittedly difficult to express \textit{empirical} preference over prospects. Indeed, within the bounds of an SDP, an agent only ever chooses between prospects originating in the same state. In Subsection \ref{subsection_cardinal} we will apply a ``veil of ignorance'' argument \cite{harsanyi1953cardinal,rawls2009theory} to enable a general comparison between prospects for normative purposes. 

\subsection{Constructive Markov preference (MP)}\label{subsection_constructiveMP}

In this section we derive an ``internal'' SDP from an environment SDP, $(S, A, T, T_0)$, so that Markov preference is satisfied with respect to the internal SDP. First, define a \textit{historical prospect} to be a pair $(y_t, \Pi_t)$ where $y_t \in \cup_n Y_n$ and $\Pi_t$ is the policy to be followed when starting in the final state of $y_t$. One should have little trouble accepting asymmetry with respect to preferences over historical prospects.

Next, define an equivalence relation on the set of all histories as follows: $y_i, y_j \in \cup_n Y_n$ are equivalent if the last states are equal, $y_{i[i]} = y_{j[j]}$, and, for all $\Pi^1$ and $\Pi^2$, $(y_i, \Pi_i^1) \succ (y_i, \Pi_i^2) \iff (y_j, \Pi_j^1) \succ (y_j, \Pi_j^2)$. Let $S'$ be the set of equivalence classes with generic element $s' = \{y_t\}$. Note that $S'$ may be uncountable even if $S$ is finite.

It follows from our construction that preferences over the prospects $(s', \Pi)$, where $s' \in S'$ and $\Pi$ is an arbitrary non-stationary policy, are independent of time $t$ and history $y_t$. Therefore, the constructed SDP, $(S', A, T', T_0)$, where $T'(s',a) \vcentcolon= T(y_{t[t]},a)$, satisfies Markov preference. 

\subsection{Cardinal utility over prospects} \label{subsection_cardinal}

We imagine a hypothetical state from which an agent chooses between lotteries (i.e., probability distributions with finite support) of prospects, denoted by $\mathscr{L}(\mathcal{P})$. We might think of this choice being made from behind a ``veil of ignorance'' \cite{rawls2009theory}. In our case, the veil is not entirely hypothetical: the agent's designer, from whom preferences are derived, is faced with a similar choice problem. The designer can instantiate the agent's internal state, which encapsulates the agent's subjective belief about the history, to be anything (albeit, the external state is outside the designer's control). Now there may be some uncertainty about which internal hardware states correspond to which histories, so that the designer is, in a sense, behind a veil of ignorance. 

We assume that strict preference ($\succ$) with respect to arbitrary prospect lotteries $\tilde{p}, \tilde{q}, \tilde{r} \in \mathscr{L}(\mathcal{P})$ satisfies: 

\begin{axiom}[Asymmetry]
	If $\tilde{p} \succ \tilde{q}$, then not $\tilde{q} \succ \tilde{p}$. 
\end{axiom}

\begin{axiom}[Negative transitivity]
	If not $\tilde{p} \succ \tilde{q}$, and not $\tilde{q} \succ \tilde{r}$, then not $\tilde{p} \succ \tilde{r}$.
\end{axiom}

\begin{axiom}[Independence]\label{axiom_independence}
	If $\alpha \in (0, 1]$ and $\tilde{p} \succ \tilde{q}$, then $\alpha\tilde{p} + (1-\alpha)\tilde{r} \succ \alpha\tilde{q} + (1-\alpha)\tilde{r}$.
\end{axiom}

\begin{axiom}[Continuity]\label{axiom_continuity}
	If $\tilde{p} \succ \tilde{q} \succ \tilde{r}$, then $\exists\ \alpha, \beta \in (0, 1)$ such that $\alpha\tilde{p} + (1-\alpha)\tilde{r} \succ \tilde{q} \succ \beta\tilde{p} + (1-\beta)\tilde{r}$.
\end{axiom}

The notation $\alpha \tilde{p} + (1-\alpha) \tilde{q}$ above represents a \textit{mixture} of prospect lotteries, which is itself a prospect lottery with a $\alpha\%$ chance of lottery $\tilde{p}$ and a $(1 - \alpha)\%$ chance of lottery $\tilde{q}$. Note that prospects are (degenerate) prospect lotteries.

We now apply the VNM Expected Utility Theorem \cite{von1953theory} to obtain a cardinal utility function over prospects. We use a version of the theorem from \inlinecite{kreps1988notes} (Theorem 5.15), restated here without proof and with minor contextual modifications:

\begin{theorem}[Expected utility theorem]\label{theorem_VNM} The binary relation $\succ$ defined on the set $\mathscr{L}(\mathcal{P})$ satisfies Axioms 1-4 if and only if there exists a function $U: P \to \mathbb{R}$ such that, $\forall\ \tilde{p}, \tilde{q} \in \mathscr{L}(\mathcal{P})$:
$$\tilde{p} \succ \tilde{q} \iff \sum_{z} \tilde{p}(z)U(z) > \sum_{z} \tilde{q}(z)U(z)$$
where the two sums in the display are over all $z \in \mathcal{P}$ in the respective supports of $\tilde{p}$ and $\tilde{q}$. Moreover, another function $U'$ gives this representation if and only if $U'$ is a positive affine transformation of $U$. 
\end{theorem}

Applying the theorem produces the cardinal utility function $U : \mathcal{P} \to \mathbb{R}$, as desired. We overload notation and define $U: \mathscr{L}(\mathcal{P}) \to \mathbb{R}$ as $U(\tilde{p}) = \sum_{z} \tilde{p}(z)U(z)$. 

We further define $U^\Pi : S \to \mathbb{R}$ as $U^\Pi(s) = U((s, \Pi)) = U(s, \Pi)$ for policy $\Pi$. Similarly, we overload $U^\Pi$ to define $U^\Pi : S \times A \to \mathbb{R}$ as $U^\Pi(s, a) = U(s, a\Pi)$.

Finally, given a lottery over states, $\tilde{s} \in \mathscr{L}(S)$, we denote the prospect lottery given by $\tilde{s}$ with fixed $\Pi$ as $(\tilde{s}, \Pi)$, and define preference ordering $\succ_{\tilde{s}}$ over policies induced by $\tilde{s}$  according to the rule $\Pi \succ_{\tilde{s}} \Omega$ if and only if $U(\tilde{s}, \Pi) = \sum_z\tilde{s}(z)U^\Pi(z) > \sum_z\tilde{s}(z)U^\Omega(z) = U(\tilde{s}, \Omega)$. We further define the shorthand $U^{\tilde{s}}: \mathbf{\Pi} \to \mathbb{R}$ as $U^{\tilde{s}}(\Pi) = U(\tilde{s}, \Pi)$.

\subsection{Rational planning}

Axioms 1-4 are classics of decision theory and have been debated extensively over the years (see footnote 1 of \inlinecite{machina1989dynamic} for some initial references). Other than asymmetry, which is natural given MP, we do not wish to argue for their merits. Rather, it is prudent to strengthen this foundation for the sequential context. Without strong normative assumptions about the structure of preferences, such as those implied by the standard optimality criteria of MDPs, one could infer little about future behavior from past experiences and learning would be impossible; see, e.g., the ``No Free Lunch'' theorem for inverse reinforcement learning \cite{armstrong2017impossibility}. The axioms and results in this subsection provide what we argue is a minimal characterization of rational planning. We begin with:

\begin{axiom}[Irrelevance of unrealizable actions] \label{axiom_irrelevance_unrealizable}
If the stochastic processes generated by following policies $\Pi$ and $\Omega$ from initial state $s$ are identical, then the agent is indifferent between prospects $(s, \Pi)$ and $(s, \Omega)$. 
\end{axiom}

A consequence of this axiom is that $U(s, a\Pi)$ is well-defined, since $U(s, a\Pi)$ is constant for all first-step policies $\pi$ with $\pi(s) = a$.

Assuming non-trivial MP, an agent will choose between prospects of the form $(s_2, \Pi_2)$ at time $t=2$, where $s_2 \sim T(s_1, a_1)$ and $\Pi_2$ is any non-stationary policy. The agent has preferences over \textit{plans} for this choice at $t=1$, which can be ascertained by restricting the $t=1$ choice set to the set $X$ of prospects of the form $(s_1, a_1\Pi)$. From a rational agent, we should demand that the restriction of $U$ to $X$, $\left.U\right|_X: \mathbf{\Pi} \to \mathbb{R}$, represent the same preference ordering over policies as $U^{T(s_1, a_1)}$. We thus assume:

\begin{axiom}[Dynamic consistency]\label{axiom_consistency}
$(s, a\Pi) \succ (s, a\Omega)$ if and only if $(T(s, a), \Pi) \succ (T(s, a), \Omega)$. 
\end{axiom}

Dynamic consistency is based on the similar axioms of \inlinecite{sobel1975ordinal} and \inlinecite{kreps1978temporal}, reflects the general notion of dynamic consistency discussed by \inlinecite{machina1989dynamic}, and might be compared to Koopmans' classic ``stationarity'' axiom \shortcite{koopmans1960stationary}. Note that we demand consistency only before and after an action has been chosen, but not before and after environmental uncertainty is resolved. That is, $(T(s, a), \Pi) \succ (T(s, a), \Omega)$ \textit{does not imply} that for all $z$ in the support of $T(s, a)$, $(z, \Pi) \succ (z, \Omega)$.

Finally, we adopt a mild version of ``impatience'', comparable to Sobel's countable transitivity axiom \shortcite{sobel1975ordinal}, but stated here in terms of utility (for clarity). Impatience can be understood as the desire to make the finite-term outcomes meaningful in light of an infinite time horizon \cite{koopmans1960stationary}. In the statement below, $\Pi_n\Omega$ is the policy that follows $\Pi$ for the first $n$ steps and $\Omega$ thereafter. 

\begin{axiom}[Horizon continuity]\label{axiom_countable_transitivity} The sequence $\{U(s, \Pi_n\Omega)\}$ converges with limit $U(s, \Pi)$.  
\end{axiom}

One might use this basic setup to prove a number of facts about rational behavior in SDPs; e.g., \inlinecite{sobel1975ordinal} uses a similar axiomatic structure to prove a policy improvement theorem alongside the next result. We restrict our analysis to three immediately relevant results (but see our comment on convergence theorems in Subsection \ref{subsection_GRL}). The first justifies our later focus on stationary policies. An \textit{optimal policy} is a policy $\Pi$ for which $(s, \Pi) \succeq (s, \Omega)$ for all $s$ and $\Omega$.

\begin{lemma} \label{lemma_optdelay}
If $\Pi$ is an optimal policy, so too is the policy $\Pi_1\Pi = (\Pi_{[0]}, \Pi_{[0]}, \Pi_{[1]}, \Pi_{[2]}, \dots)$ formed by delaying $\Pi$ one step in order to act according to $\Pi_{[0]}$ for that step. 
\end{lemma}

\begin{proof}
Consider any state $s$. For each state $z$ in the support of $\Pi(s)$, $(z, \Pi) \succeq (z, \Pi_{[1:]})$ (because $\Pi$ is optimal) so that $(T(s, \Pi(s)), \Pi) \succeq (T(s, \Pi(s)), \Pi_{[1:]})$. By dynamic consistency, this implies $(s, \Pi_1\Pi) \succeq (s , \Pi)$.  
\end{proof}

\begin{theorem} \label{theorem_optstat}
If there exists an optimal policy $\Pi$, there exists an optimal stationary policy $\pi$.
\end{theorem}

\begin{proof}[Proof]
Put $\pi = \Pi_{[0]}$. By repeated application of Lemma \ref{lemma_optdelay} we have $\pi_n\Pi \succeq \Pi$ for all $n > 0$. It follows from horizon continuity that $\pi \succeq \Pi$. 
\end{proof}

The next result is somewhat similar Lemma 4 of \inlinecite{kreps1978temporal}, but with a recursive, affine formulation. Note that the proof does not use horizon continuity. 

\begin{theorem}[Bellman relation for SDPs]\label{theorem_functional_form}
	There exist $\mathscr{R}: S \times A \to \mathbb{R}$ and $\Gamma: S \times A \to \mathbb{R}^+$ such that for all $s, a, \Pi$, $$U(s, a\Pi) = \mathscr{R}(s, a) + \Gamma(s, a)\mathbb{E}_{s' \sim T(s, a)}[U(s', \Pi)].$$
\end{theorem}

\begin{proof}
	Fix $s$ and $a$. Dynamic consistency ensures that $U^{T(s, a)} = \mathbb{E}_{s' \sim T(s, a)}[U(s', \Pi)]$ represents the same preferences as the restriction of $U$ to the space $X$ of prospects of the form $(s, a\Pi)$. Preferences are cardinal because $\Pi$ may be stochastic, so that prospects in $X$ are prospect lotteries (i.e., $X \subset \mathscr{L}(\mathcal{P})$), $X$ is closed under mixtures (i.e., $\tilde{p},\tilde{q} \in X \implies \alpha\tilde{p} + (1-\alpha)\tilde{q} \in X$, $\forall \alpha \in [0, 1]$), and Axioms 1-4 apply to prospect lotteries in $X$. Therefore, by the restriction of Theorem 1 to $X$, $\left.U\right|_X$ and $U^{T(s, a)}$ are related by the positive affine transformation $\left.U\right|_X = \alpha + \beta U^{T(s, a)}$ for some $\alpha \in \mathbb{R}, \beta \in \mathbb{R}^+$. Define $\mathscr{R}(s, a) = \alpha$, $\Gamma(s, a) = \beta$. Since $s$ and $a$ were arbitrary, the result follows. 
\end{proof}

The final result of this subsection will be used to prove the value-utility relation of Subsection \ref{subsection_utility_value_bounds} and is useful for analyzing convergence of RL algorithms (see discussion in Subsection \ref{subsection_GRL}). It uses two new assumptions. First, we assume $\vert S \vert = n$ is finite. This allows us to define vectors $\textbf{u}^\Pi$ and $\textbf{r}^\pi$ so that their $i$th components equal $U(s_i, \Pi)$ and $\mathscr{R}(s_i, \pi(s_i))$, respectively. Further, we define diagonal matrix $\bm{\Gamma}^\pi$ whose $i$th diagonal entry is $\Gamma(s_i, \pi(s_i))$ and transition matrix $\textbf{T}^\pi$ whose $ij$th entry is $T(s_i, \pi(s_i))(s_j)$. Second, we assume the set $\{\textbf{u}^\Pi : \Pi \in \bm{\Pi}\}$ spans $\mathbb{R}^n$. This latter assumption is sensible given that $\vert \{\textbf{u}^\Pi\} \vert \gg n$ (one can also hypothetically alter $S$ or $A$ to make it true). We have:

\begin{theorem}[Generalized successor representation]\label{theorem_matrix_limit}
	If $\vert S \vert = n$ and $\text{span}(\{\normalfont{\textbf{u}}^\Pi\}) = \mathbb{R}^n$,  $\lim_{n\to\infty}(\bm{\Gamma}^\pi \mathbf{T}^\pi)^n = \mathbf{0}$, so that $(\mathbf{I} - \bm{\Gamma}^\pi\mathbf{T}^\pi)^{-1} = \mathbf{I}+(\bm{\Gamma}\mathbf{T})^1+(\bm{\Gamma}\mathbf{T})^2 + \dots$ is invertible.
\end{theorem}

\begin{proof}
	Using $a\Pi = \pi_n\Omega$ in the vector form of Theorem \ref{theorem_functional_form}, and expanding the recursion $n-1$ steps gives:
	\small
	\begin{equation*}
	\begin{split} \mathbf{u}^{\pi_n\Omega} &= \mathbf{r}^\pi + \bm{\Gamma}^\pi\mathbf{T}^\pi\mathbf{u}^{\pi_{n-1}\Omega} \\
	&= \mathbf{r}^\pi(\mathbf{I}+(\bm{\Gamma}\mathbf{T})^1 + \dots + (\bm{\Gamma}\mathbf{T})^{n-1}) + (\bm{\Gamma}\mathbf{T})^n\mathbf{u}^\Omega
	\end{split}
	\end{equation*}
	\normalsize
	where the superscripts on $\bm{\Gamma}^\pi$ and $\mathbf{T}^\pi$ were dropped for convenience. Similarly, using $a\Pi = \pi$, we have:
	\small
	$$\mathbf{u}^{\pi} = \mathbf{r}^\pi(\mathbf{I}+(\bm{\Gamma}\mathbf{T})^1 + \dots + (\bm{\Gamma}\mathbf{T})^{n-1}) + (\bm{\Gamma}\mathbf{T})^n\mathbf{u}^\pi.$$
	\normalsize
	Subtracting the second from the first gives 
	$\mathbf{u}^{\pi_n\Omega} - \mathbf{u}^{\pi} = (\bm{\Gamma}^\pi\mathbf{T}^\pi)^n(\mathbf{u}^\Omega - \mathbf{u}^\pi).$
	By horizon continuity, the left side goes to $\mathbf{0}$ as $n\to\infty$ for all $\pi$ and $\Omega$. Since $\{\textbf{u}^\Omega\}$ spans $\mathbb{R}^n$, so too does $\{\mathbf{u}^\Omega - \mathbf{u}^\pi\}$. It follows that $(\bm{\Gamma}^\pi\mathbf{T}^\pi)^n \to \mathbf{0}$, for otherwise, there exists $\textbf{x} \neq \mathbf{0}$ for which $(\bm{\Gamma}^\pi\mathbf{T}^\pi)^n \textbf{x} \not\to \mathbf{0}$, but $\textbf{x}$ can be expressed as a linear combination of the spanning set, which leads to a contradiction. That $(\mathbf{I} - \bm{\Gamma}^\pi\mathbf{T}^\pi)$ is invertible follows from the well known matrix identity (\inlinecite{kemeny1960finite} \S 1.11.1).
\end{proof}

\subsection{Preference structure of MDPs}\label{subsection_pref_struct_of_MDPs}

The value function of an MDP induces preferences over $\mathscr{L}(\mathcal{P})$ when treated as a utility function on $\mathcal{P}$; i.e., according to the rule $\tilde{p} \succ \tilde{q}$ if $\sum_{(s, \Pi)} \tilde{p}((s, \Pi))V^\Pi(s) > \sum_{(s, \Pi)} \tilde{q}((s, \Pi))V^\Pi(s)$. We have that:

\begin{theorem}\label{theorem_consistency}
Preferences induced by the value function of an MDP in continuous settings, with fixed $\gamma < 1$, and in episodic settings, with $\gamma = 1$, satisfy Axioms 1-7.
\end{theorem} 

\begin{proof}
Axioms 1-4 follow from the necessity part of Theorem \ref{theorem_VNM}. Axioms 5 and 6 are obvious, as is Axiom 7 in the episodic case. Axiom 7 is true in the continuous case because bounded $R$ and $\gamma < 1$ imply that the total contribution of rewards received after $n$ steps goes to $0$ as $n \to \infty$.
\end{proof}

\begin{corollary}Theorem 2 applies to MDPs (this is well known; see, e.g., Theorem 6.2.9(b) of \inlinecite{puterman2014markov}).
\end{corollary}

If the converse of Theorem \ref{theorem_consistency} were true (i.e., Axioms 1-7 implied fixed $\gamma$), this would validate the use of the MDP as the default model for RL. Unfortunately, Theorem \ref{theorem_functional_form} is the closest we can get to such a result; it is not hard to see that the utilities in the Cliff Example of Section \ref{section_motivation}---which are inconsistent with any fixed $\gamma$---are consistent with the axioms (this is illustrated in the Supplement). 

\subsection{$\bm{\Gamma(s,a) > 1}$}\label{subsection_gamma}

A key feature of our Theorem \ref{theorem_functional_form} representation is that, but for Axiom \ref{axiom_countable_transitivity}, which demands long run discounting, $\Gamma$ is unconstrained. As noted in Section \ref{section_related}, this differentiates our result from the otherwise identical recursions of \inlinecite{meyer1976preferences} and \inlinecite{epstein1983stationary}, which bound the discount factor by 1. It also suggests a more general form of value function than the proposals of \inlinecite{white2017UnifyingTS}, \inlinecite{silver2017ThePE} and \inlinecite{sutton2011horde}, which adopt a variable $\gamma$, but constrain it to $[0,1]$. 

A natural question is whether our axioms can (and \textit{should}) be extended to provide a normative justification for a constrained discount factor---in particular, should we apply the same argument as Meyer and Epstein? Let us first understand their result. Meyer and Epstein both consider trajectories as primitive objects of preference, so that each trajectory has a well-defined utility. Further, they assume that the utilities of trajectories are bounded, that the domain of preference includes arbitrary trajectories, and that at least one trajectory $\tau$ has non-zero utility (which is true so long as the agent is not indifferent between all trajectories). Now suppose there exists $(s, a)$ for which $\Gamma(s, a) > 1$. It must be that $\mathscr{R}(s, a) = 0$, else $U((s, a, s, a, \dots))$ does not exist. But then the sequence of utilities $\{U((s, a, \tau)), U((s, a, s, a, \tau)), \dots\}$ formed by repeatedly prepending $(s, a)$ to $\tau$ diverges. This contradicts the assumptions and it follows that $\Gamma(s, a) \leq 1$ for all $(s, a)$. 

We cannot apply their argument directly. Trajectories are not primitive objects in our framework, and can only be valued by successively unrolling the Theorem \ref{theorem_functional_form} Bellman relation and summing the ``discounted'' rewards along the trajectory. However, there is no guarantee that this sum converges (see Supplement for an example). But even if we were to insist that, in addition to Axioms 1-7, all plausible trajectories have bounded utilities, the Meyer-Epstein argument would still not apply and $\Gamma(s, a) > 1$ would still be possible. The reason is that \textit{arbitrary} trajectories like $(s, a, s, a, \dots)$ are not necessarily \textit{plausible}---if these trajectories are not in the support of any prospect, it does not make sense to demand that the agent express preference with respect to them, even hypothetically speaking. For instance, it does not make sense for a chess agent to evaluate the impossible sequence of moves (\texttt{e4}, \texttt{e4}, \texttt{e4}, \dots). For this reason, we argue that the Meyer-Epstein approach to sequential rationality is too restrictive, and that we should allow $\Gamma(s, a) > 1$, so long as there is long run discounting in accordance with Axiom 7.

\subsection{The optimizing MDP}\label{subsection_optimizing}

Although the fixed $\gamma$ MDP is not expressive enough to model all preference structures considered rational by our framework, it may still serve as a model thereof. In this subsection we prove the existence of an MDP whose optimal value function equals the optimal utility function. As a preliminary step, let us restate without proof a classic result for MDPs (e.g., Theorem 6.2.6 of \inlinecite{puterman2014markov}):

\begin{lemma}[Bellman optimality]\label{lemma_bellman_opt}
Given an MDP, policy $\pi$ is optimal with respect to $V$ if and only if, $\forall s \in S$,
\begin{equation*}V^\pi(s) = \arg\max_{a \in A}(R(s, a) + \gamma\mathbb{E}[V^\pi(s')]).\end{equation*}
\end{lemma}

We also define $U^* = U^{\pi^*}$, $V^* = V^{\pi^*}$ and $Q^* = Q^{\pi^*}$, and recall that $U$ is overloaded for domains $S$ and $S \times A$.

\begin{theorem}[Existence of optimizing MDP]\label{theorem_existence}
Given an SDP with utility $U$ over prospects and optimal stationary policy $\pi^*$, for all $\gamma \in [0, 1)$, there exists a \textit{unique} ``optimizing MDP'' that extends the SDP with discount factor $\gamma$ and reward function $R$ such that $\pi^*$ is optimal with respect to $V$, and has corresponding optimal $V^* = U^*$ and $Q^* = U^*$.
\end{theorem}

\begin{proof}
Put $R(s, a) = U^*(s, a) - \gamma \mathbb{E}[U^*(s')]$. Then:
\begin{align*}
U^*(s_t) &= R(s_t, \pi^*(s_t)) + \gamma\mathbb{E}[U^*(s_{t+1})]\\
&= \mathbb{E}\left[\sum_{t=0}^\infty \gamma^t R(s_t, \pi^*(s_t))\right] = V^*(s_t)
\shortintertext{and:}
U^*(s, a) &= R(s, a) + \gamma \mathbb{E}(V^*(s_{t+1})) = Q^*(s, a).
\end{align*}
Since $V^*(s) = U^*(s, \pi^*(s_t)) \geq U^*(s, a) = Q^*(s, a)$, $\forall a \in A$, it follows from Lemma \ref{lemma_bellman_opt} that $\pi^*$ is optimal with respect to $V$. For uniqueness, suppose that $V^* = U^*$ and $Q^* = U^*$ are optimal; then by definition of $Q^*$, we have $U^*(s,a) = R(s, a) + \gamma \mathbb{E}[U^*(s')]$, so that $R(s, a) = U^*(s, a) - \gamma \mathbb{E}[U^*(s')]$ as above.
\end{proof}

A consequence of Theorem \ref{theorem_existence} is that the inverse reinforcement learning problem \cite{ng2000algorithms} is solvable (in theory); i.e., there exists a reward function that can \textit{explain} any behavior satisfying Axioms 1-7 as being the solution to an MDP. This same consequence follows from Theorem 3 of \cite{ng2000algorithms}, which characterizes the \textit{set} of reward functions under which some observed behavior is optimal. Our theorem differs from that of Ng \& Russell in that it produces a \textit{unique} solution for completely specified preferences, whereas the theorem of Ng \& Russell produces a \textit{set} of solutions for partially specified preferences.

\subsection{Relating value to utility}\label{subsection_utility_value_bounds}

Although $V^* = U^*$ in the optimizing MDP, the Cliff Example tells us that, in general, $V^\pi \neq U^\pi$. This is a potentially serious problem because an agent may \textit{never} find the optimal policy. Indeed, humans are the only general purpose agents we know of, and our bounded rationality is well documented \cite{simon1972theories,tversky1986rational}. It is possible that general purpose objective preferences are so complex that \textit{all} intelligent agents---present or future; human, artificial or alien---are so bounded to suboptimality. 

Nevertheless, a natural hypothesis is that the closer $V^\pi$ is to $V^*$---i.e., the better an agent performs in its approximate model of $U$---the better off it will tend to be. There is a sense in which this is true, at least for finite $\vert S \vert$. Recalling the vector notation defined for Theorem \ref{theorem_matrix_limit}, defining $\mathbf{v}^\pi$ accordingly, noting that $\mathbf{u}^* = \mathbf{v}^*$, and setting $\Epsilon^\pi = \bm{\Gamma}^\pi - \gamma\mathbf{I}$,  we have:

\begin{theorem}\label{theorem_utility_value}
	In the optimizing MDP (for finite $\vert S \vert$) :
	\small
	\begin{equation*}
	\begin{split}
	\mathbf{u}^\pi &= \mathbf{u}^* - (\mathbf{I} - \bm{\Gamma}^\pi\mathbf{T}^\pi)^{-1}
	(\mathbf{I} - \gamma\mathbf{T}^\pi)(\mathbf{v}^* - \mathbf{v}^\pi)\\
	&= \mathbf{v}^\pi - (\mathbf{I} - \bm{\Gamma}^\pi\mathbf{T}^\pi)^{-1}\Epsilon^\pi\mathbf{T}^\pi(\mathbf{v}^* - \mathbf{v}^\pi).\\
	\end{split}
	\end{equation*}
	\normalsize
\end{theorem}

\begin{proof}
	The Bellman relation provides four equations:
	\small
	\begin{multicols}{2}
		\noindent
		\begin{align}
		\mathbf{u}^\pi &= \mathbf{r}^\pi + \bm{\Gamma}^\pi\mathbf{T}^\pi\mathbf{u}^\pi\label{eq_1}\\
		\mathbf{u}^{a*} &= \mathbf{r}^\pi + \bm{\Gamma}^\pi\mathbf{T}^\pi\mathbf{u}^{*}\label{eq_2}
		\end{align}
		\begin{align}
		\mathbf{v}^\pi &= \bm{r}^\pi + \gamma\mathbf{T}^\pi\mathbf{v}^\pi\label{eq_3}\\
		\mathbf{u}^{a*} &= \bm{r}^\pi + \gamma\mathbf{T}^\pi\mathbf{v}^*\label{eq_4}
		\end{align}
	\end{multicols}
	\normalsize
	\noindent where $\bm{r}^\pi$ is the vector representing $R(s, \pi(s))$. The first equality of the theorem follows by computing (2) - (4), isolating $\mathbf{r}^\pi - \bm{r}^\pi$, substituting it into (1) - (3), adding $\mathbf{v}^* - \mathbf{v}^*$ ($= \mathbf{0}$) to one side, and rearranging, given that $(\mathbf{I} - \bm{\Gamma}^\pi\mathbf{T}^\pi)$ is invertible by Theorem \ref{theorem_matrix_limit}. 
	
	The second equality follows by expanding $(\mathbf{I} - \bm{\Gamma}\mathbf{T})^{-1} = \mathbf{I} + \bm{\Gamma}\mathbf{T} + (\bm{\Gamma}\mathbf{T})^2 + \dots = \mathbf{I} + (\mathbf{I} - \bm{\Gamma}\mathbf{T})^{-1}\bm{\Gamma}\mathbf{T}$, and writing:
	\small
	\begin{equation*} \scalebox{0.89}[1]{$
	\begin{aligned}
	\mathbf{u}^\pi &= \mathbf{u}^* - (\mathbf{I} + (\mathbf{I} - \bm{\Gamma}\mathbf{T})^{-1}\bm{\Gamma}\mathbf{T})
	(\mathbf{I} - \gamma\mathbf{T}^\pi)(\mathbf{v}^* - \mathbf{v}^\pi)\\
	&= \mathbf{v}^\pi - [(\mathbf{I} - \bm{\Gamma}\mathbf{T})^{-1}\bm{\Gamma} - \gamma(\mathbf{I} + (\mathbf{I} - \bm{\Gamma}\mathbf{T})^{-1}\bm{\Gamma}\mathbf{T})]\mathbf{T}(\mathbf{v}^* - \mathbf{v}^\pi).\\
	&= \mathbf{v}^\pi - [(\mathbf{I} - \bm{\Gamma}\mathbf{T})^{-1}\bm{\Gamma} - (\mathbf{I} - \bm{\Gamma}\mathbf{T})^{-1}\gamma\mathbf{I}]\mathbf{T}(\mathbf{v}^* - \mathbf{v}^\pi).\\
	&= \mathbf{v}^\pi - (\mathbf{I} - \bm{\Gamma}^\pi\mathbf{T}^\pi)^{-1}\Epsilon^\pi\mathbf{T}^\pi(\mathbf{v}^* - \mathbf{v}^\pi)
	\end{aligned}
	$}\end{equation*}
	\normalsize
	where we again use the identity $\mathbf{I} + (\mathbf{I} - \bm{\Gamma}\mathbf{T})^{-1}\bm{\Gamma}\mathbf{T} = (\mathbf{I} - \bm{\Gamma}\mathbf{T})^{-1}$ in the third line.	
\end{proof}

It is worth examining the factors of the product $(\mathbf{I} - \bm{\Gamma}^\pi\mathbf{T}^\pi)^{-1}\Epsilon^\pi\mathbf{T}^\pi(\mathbf{v}^* - \mathbf{v}^\pi)$. The final factor $(\mathbf{v}^* - \mathbf{v}^\pi)$ has non-negative entries and tells us that the difference between $\mathbf{u}^\pi$ and $\mathbf{v}^\pi$ is a linear function of the agent's regret in the optimizing MDP, which supports our hypothesis that better approximated performance is correlated with better objective performance. The factors $(\mathbf{I} - \bm{\Gamma}^\pi\mathbf{T}^\pi)^{-1}$ and $\mathbf{T}^\pi$ also have all non-negative entries (to see this, write the former as an infinite sum). Thus, the sign of the error depends on the entries of $\Epsilon^\pi = \bm{\Gamma}^\pi - \gamma\mathbf{I}$. Unfortunately, there is no way to guarantee that $\Epsilon^\pi$ has all negative entries, which would guarantee that value estimates in optimizing MDP are pessimistic, since $\Gamma(s, a)$ may be greater than 1 at some $(s, a)$.

It is emphasized that Theorem \ref{theorem_utility_value} does not address the preference reversal problem observed in the Cliff Example. Even though $\mathbf{u}^* - \mathbf{u}^\pi$ is related to $\mathbf{v}^* - \mathbf{v}^\pi$ by linear transformation, the relationship is not monotonic (entry-wise, or in norm). Preferences over suboptimal prospects implied by the optimizing MDP's value function may very well be reversed.

\section{Discussion} \label{section_GRL}

\subsection{General reinforcement learning (GRL)} \label{subsection_GRL}

Our analysis suggests that, from a normative perspective, the MDP may not be sufficiently expressive. Two approaches for representing non-MDP preference structures come to mind. First, one can adopt the Theorem \ref{theorem_functional_form} representation and use an ``MDP+$\Gamma$'' model of the environment that defines both an external reward signal $R$ and an external anticipation function $\Gamma$. This is the approach of White's RL task formalism (2017), which, motivated by practical considerations, proposes to specify tasks using a transition-based discount factor. This is also the approach taken by \inlinecite{silver2017ThePE}. The theoretical analysis of \inlinecite{white2017UnifyingTS} covers convergence of standard algorithms for the case of $\Gamma \leq 1$. Although we defer a rigorous analysis of convergence in our more general case, we note that Theorem \ref{theorem_matrix_limit}, which guarantees that $\bm{\Gamma}^\pi\mathbf{T}^\pi$ (corresponding to White's $\textbf{P}_{\pi, \gamma}$) has maximum eigenvalue less than 1, makes our setting easier to analyse (we get White's Lemma 3 for free).

Second, one might avoid the use of a single, monolithic MDP to model global preferences, and instead think about ways of coordinating many specialized MDPs. One way to do so is hierarchically, as in hierarchical RL. The idea is that it is, or at least should be, easier to express accurate preference at the level of goals (i.e., without incurring preference reversal between suboptimal goals) than at the level of fine-grained prospects. So given a set of goals $G$, an agent can decompose its preferences into two stages: first pick $g\in G$, and then optimize a $g$-specialized MDP, $\bm{M}_g$, to pick fine-grained actions. \inlinecite{kulkarni2016hierarchical} provides an example of how this might be done. Note that all $\bm{M}_g$s share world dynamics, so that only $R$ (and optionally $\gamma$) change with $g$. 

As these approaches allow an agent to represent more general preference structures than the standard RL framework, one might term them \textit{general reinforcement learning} (GRL). This is a bit of a misnomer, however, as the RL problem, broadly framed, includes GRL. 

\subsection{Inverse and preference-based RL}

Our work leads to a generalization of the inverse reinforcement learning (IRL) problem. Rather than asking, ``given the observed behavior, what reward signal, if any, is being optimized?'' \cite{russell1998learning}, our work suggests asking: \textit{given the observed behavior, what utility function (parameterized by $R$ and $\Gamma$), if any, is being optimized?} Future work should explore whether this produces empirical improvements.

Preference-based RL (PBRL) \cite{wirth2017survey} side steps reward function design by learning directly from human preferences. Although a wide variety of algorithms fall under this general umbrella, PBRL ``still lacks a coherent framework'' \cite{wirth2017survey}. In particular, the \textit{object} of preference has varied widely between different algorithms. For instance, studies in PBRL have seen preferences expressed over actions given a state \cite{griffith2013policy}, over entire policies \cite{akrour2011preference}, over complete trajectories \cite{wilson2012bayesian}, and over partial trajectories \cite{christiano2017deep}. Yet, per the discussion in Subsection \ref{subsection_prefs_over_prospects}, none of these objects satisfy Axiom 1 (asymmetry) without additional assumptions, which are not always explicitly stated or analyzed. Our work is relevant to PBRL in that it proposes a basic object---the prospect $(s, \Pi)$---over which (objective) preferences are asymmetric (given MP). Our axiomatic framework falls well short of the coherent foundation sought by Wirth, however, as there is no obvious way in which preferences over abstract prospects can be empirically expressed. 

\subsection{Generalized successor representation}

An interesting connection between the MDP and the MDP-$\Gamma$ is that, in both, a policy-specific value function can be decomposed into the product of ``discounted expected future state occupancies'' and rewards. In the finite case, the former factor is represented by the matrix $\mathbf{S} = (\mathbf{I} - \bm{\Gamma}\mathbf{T})^{-1}$ (see Theorem \ref{theorem_matrix_limit}), so that $\mathbf{v} = \mathbf{S}\mathbf{r}$. When $\Gamma = \gamma$, $\mathbf{S}$ is the well known successor representation (SR) \cite{dayan1993improving}. 

What makes the SR interesting here is that it seems to solve some of the problems posed by the abstract anticipation function $\Gamma$. First, $\Gamma$ is sensitive to the discretization of time (for the same reason annual interest rates are larger than monthly ones). Second, small changes in the average $\Gamma$ can result in large changes in value (increasing $\gamma$ from 0.990 to 0.991 increases the value of a constant positive perpetuity by over 10\%). By contrast, the entries of $\mathbf{S}$---despite its opaque formula---provide an interpretable and time-scale invariant measure of causality \cite{pitis2018source}. Changes in $\mathbf{S}$ impact $\mathbf{v}$ in proportion to $\mathbf{r}$, and there is even evidence to suggest that the humans utilize the SR to cache multi-step predictions \cite{momennejad2017successor}. For these reasons, it may be easier and more effective to elicit and use values of $\mathbf{S}$, representing long run accumulations of the $\Gamma$ function, than individual values of $\Gamma$, whether for GRL, IRL or PBRL.

\section{Conclusion} \label{section_conclusion}

We have provided normative justification for a departure from the traditional MDP setting that allows for a state-action dependent discount factor that can be greater than 1. Future work should empirically test whether our proposal has practical merit and confirm that classic convergence results extend to our more general setting. 

\small
\bibliography{rationality_final_with_supplement}
\bibliographystyle{aaai}
\normalsize

\vfill

\begin{figure*}[t]
\includegraphics[width=\textwidth]{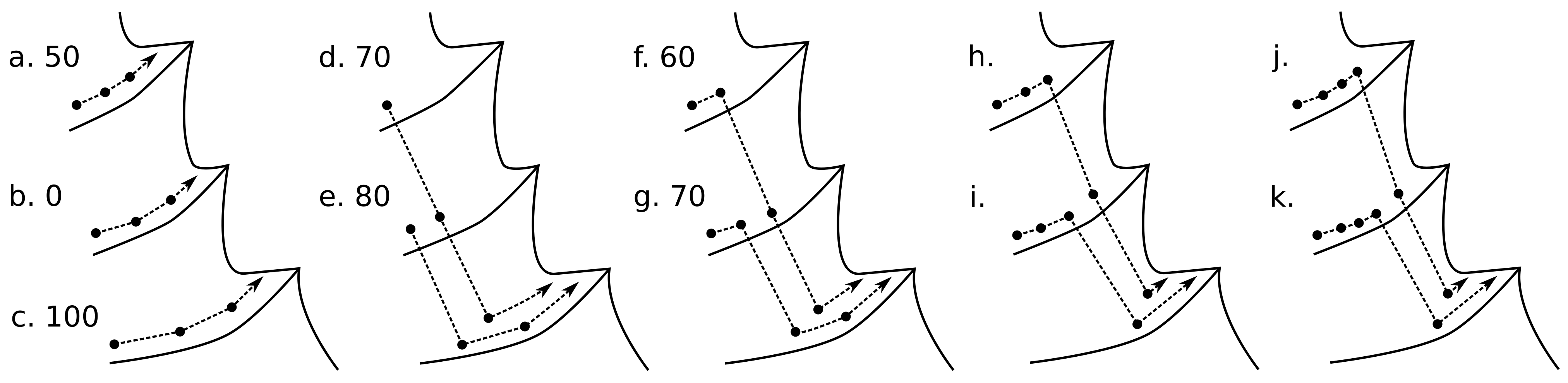}
\caption{Extension of Figure \ref{fig_cliffMDP} to illustrate more paths.}
\label{fig_cliffMDP1}
\end{figure*}

\section*{Supplement}

In this supplement we use the Cliff Example to illustrate concretely three claims made in the paper. The path labels in Figure \ref{fig_cliffMDP1} and the state and state-action labels in Figure \ref{fig_cliffMDP2} are referenced throughout. 

\boldmath\subsection*{Preference reversal occurs for all $\gamma \in [0, 1)$}\unboldmath

This claim is made in Section \ref{section_motivation}. Given the values of paths d, e, f and g in Figure \ref{fig_cliffMDP1}, we have:
\small
$$ V(g) = R(\text{MM}) + \gamma V(e), \text{\ \ so that\ \ } R(\text{MM}) = 70 - \gamma 80,$$
\normalsize
and
\small
$$ V(f) = R(\text{HH}) + \gamma V(d), \text{\ \ so that\ \ } R(\text{HH}) = 60 - \gamma 70.$$
\normalsize
It follows that $R(\text{MM}) > R(\text{HH})$ for all $\gamma \in [0, 1)$ since $R(\text{MM}) - R(\text{HH}) = 10 -\gamma10$ is positive. 

\subsection*{The Cliff Example satisfies Axioms 1-7}

This claim is made in Subsection \ref{subsection_pref_struct_of_MDPs}. The following $\mathscr{R}$ and $\Gamma$ provide one consistent Theorem \ref{theorem_functional_form} representation (not unique) and imply the following utilities for paths h-k:

\small
\begin{multicols}{2}
	\begin{tabular}{@{}ccc@{}}
	  $(s, a)$ & $\mathscr{R}(s, a)$ & $\Gamma(s, a)$ \\
	  \hline
	  LL & 10 & 0.9 \\
	  ML & -10 & 0.9 \\
	  MM & 0 & 0.875 \\
	  HM & -2 & 0.9 \\
	  HH & 25 & 0.5 \\
	\end{tabular}
	\begin{center}
	\begin{tabular}{@{}cc@{}}
	  path & $U(\text{path})$ \\
	  \hline
	  h & 55 \\
	  i & 61.25  \\
	  j & 52.5  \\
	  k & 53.59375  \\
	\end{tabular}
	\end{center}
\end{multicols}
\normalsize

If other paths (not shown) are given utilities consistent with Theorem $\ref{theorem_functional_form}$, and we assume the utility of all lotteries (none shown) are computed as expected values, Axioms 1-4 and dynamic consistency hold (by the necessity part of Theorem \ref{theorem_VNM}, and by the fact that positive affine transformations represent the same preferences, respectively). Axiom \ref{axiom_irrelevance_unrealizable} is obviously satisfied. Finally, notice that each step that path d is delayed (paths f, h, j...) brings the utility closer to path a. Similarly, delaying path e (paths g, i, k...) brings the utility closer to that of path b. This is true in general because $\Gamma < 1$, so that the contribution of future rewards to utility goes to zero and horizon continuity is satisfied.

\boldmath\subsection*{$\Gamma(s, a) > 1$ is consistent with Axioms 1-7}\unboldmath

This claim is made in Section \ref{section_related} and Subsection \ref{subsection_gamma} as one of the factors differentiating our framework from those of \inlinecite{meyer1976preferences} and \inlinecite{epstein1983stationary}, in addition to that of \inlinecite{white2017UnifyingTS}. To see that this is possible, suppose that action HH were stochastic rather than deterministic: although the agent attempts to stay on the high path, half of the time the agent trips and slides down to the middle path. Using the $\mathscr{R}$ and $\Gamma$ shown in the table to the left, but setting $\Gamma(\text{HH}) = 1.2$, results in an MDP-$\Gamma$ that has bounded utility over prospects and satisfies Axioms 1-7. Even though the ``discounted'' reward along path a is unbounded, the agent \textit{almost surely} cannot stay on path a, and the expectation of any policy that tries to stay on path a exists. In particular, an agent that tries to stay on path a, and resorts to the policy that follows path e if it trips (call this prospect $X$), has utility
$U(X) = 25 + 1.2(0.5\times U(X) + 0.5\times 80)$, so that $U(X) = 182.5$, which is finite. Since the utilities of all prospects exist, the fact that this revised MDP-$\Gamma$ is consistent with axioms 1-7 follows by an argument similar to that used for the original Cliff Example. 

\begin{figure}[t]
\includegraphics[width=\columnwidth]{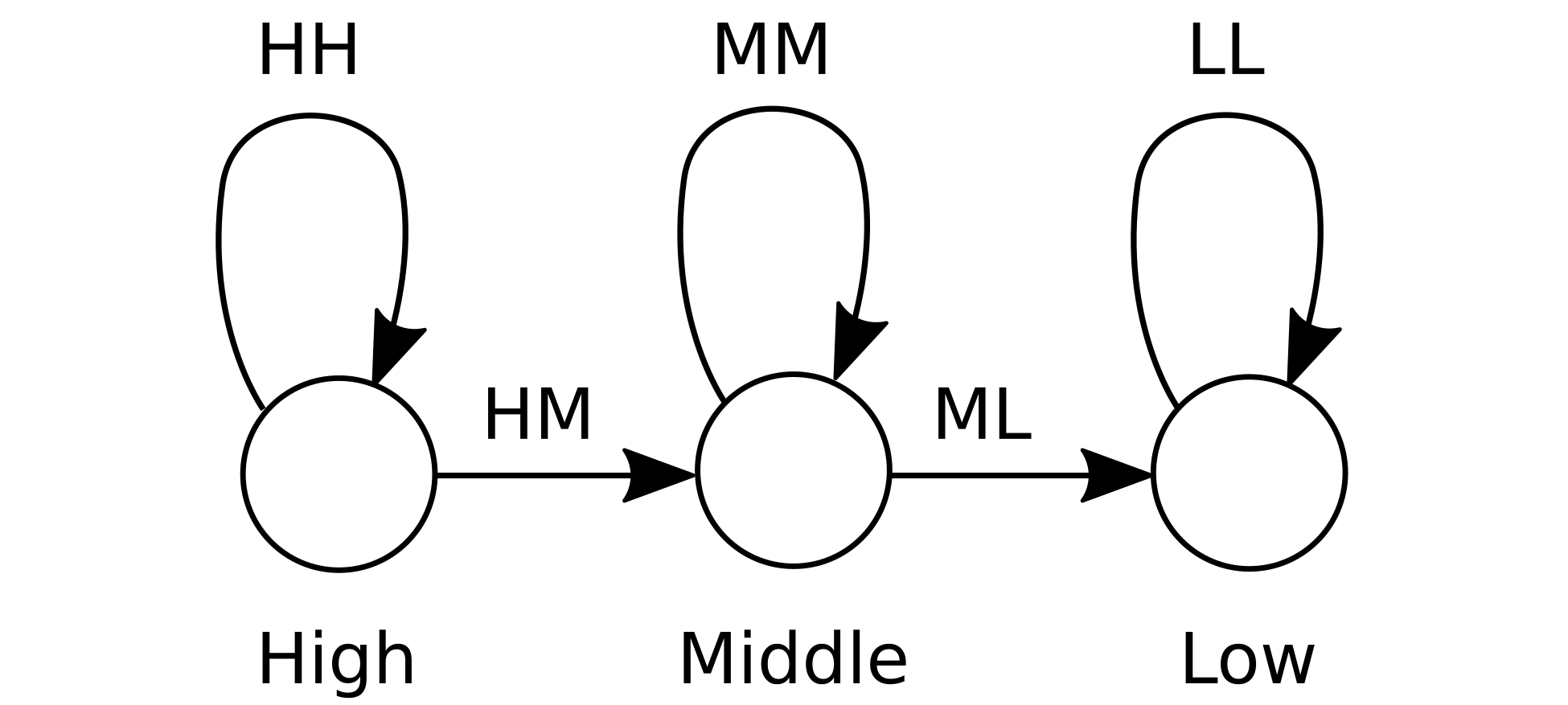}
\caption{SDP representation of Cliff Example.}
\label{fig_cliffMDP2}
\end{figure}

\end{document}